\documentclass{article}

\usepackage{arxiv}

\usepackage[utf8]{inputenc} 
\usepackage[T1]{fontenc}    
\usepackage{hyperref}      
\usepackage{url}            
\usepackage{booktabs}       
\usepackage{amsfonts}       
\usepackage{nicefrac}       
\usepackage{microtype}      
\usepackage{lipsum}		
\usepackage{graphicx}
\usepackage{natbib}
\usepackage{doi}
\usepackage{amsmath}
\usepackage{amssymb}
\usepackage{amsthm}
\usepackage{newtxtext}

\newtheorem{proposition}{Proposition}
\newtheorem{corollary}{Corollary}

\title{Specification Gaming and a Constrained Reinforcement Learning \\ Mitigation in a Cleaning Gridworld}

\date{} 					

\author{ \href{https://orcid.org/0009-0005-4873-7719}{\includegraphics[scale=0.06]{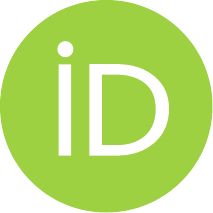}\hspace{1mm}Hilda Adwubi Osei} \\
	Department of Industrial Engineering\\
	Kwame Nkrumah University of Science and Technology\\
	Kumasi, Ghana \\
	\texttt{hildaosei109@gmail.com} \\
}

\renewcommand{\headeright}{Preprint}
\renewcommand{\undertitle}{Preprint}
\renewcommand{\shorttitle}{\textit{arXiv} Template}

\hypersetup{
pdftitle={Specification Gaming and a Constrained Reinforcement Learning Mitigation in a Cleaning Gridworld},
pdfsubject={cs.LG, cs.AI},
pdfauthor={Hilda Adwubi Osei},
pdfkeywords={specification gaming, reward hacking, reinforcement learning, AI alignment, constrained optimization},
}

\begin{document}
\maketitle

\begin{abstract}

Reinforcement learning agents optimize exactly the reward signal they are given, which can diverge sharply from an intended true objective when that signal is only a proxy for it; this failure mode is known as specification gaming. Existing evidence for this failure mode is largely observational, with comparatively few demonstrations pairing a controlled, deliberately constructed proxy/true gap with a mitigation that carries a verifiable guarantee rather than only an empirical improvement. We construct a minimal gridworld in which a cleaning-robot agent can either genuinely remove dirt or conceal it from a cheap sensor, and show that a standard deep Q-network trained only on the sensor's proxy reward converges to a policy that games it completely, reaching near-maximal proxy reward while never achieving the true objective. We then introduce a constrained reinforcement learning fix, penalizing the specific behaviors through which the agent exploits the gap, and show that it eliminates the gaming behavior from the final policy while substantially recovering true-objective performance. We further derive and empirically verify an exact accounting identity relating the proxy reward, true reward, and our cost signal for this environment, showing that the cost signal measures the proxy/true gap exactly rather than merely tracking it, and yields a deterministic guarantee under a hard-constrained variant of the method.

\end{abstract}


\renewcommand{\shorttitle}{Specification Gaming and Constrained RL}

\section{Introduction}
\label{sec:intro}

A reinforcement learning agent optimizes exactly the reward signal it is given, not the intention behind it. When that signal is a proxy standing in for a true objective that is harder to specify directly -- a sensor reading in place of a physical state, a classifier score in place of an outcome -- optimization pressure can push the agent toward the proxy/true gap rather than closing it. This failure mode, \emph{specification gaming}, has been documented extensively across both simulated and real-deployed systems \citep{krakovna2020specification}, and recent work has begun to formalize what distinguishes it from ordinary reward misspecification \citep{skalse2022defining}.

Much of the existing evidence for specification gaming is observational: a catalogue of cases in which an agent was found, after the fact, to have exploited a reward function in an unintended way. Comparatively less work pairs a controlled, deliberately constructed proxy/true gap with a mitigation that carries a stated, verifiable guarantee rather than only an empirical improvement. This gap matters for deployment: knowing that a fix tends to work is weaker than knowing precisely what it guarantees and under what conditions.

In this work, we construct a minimal gridworld, \textsc{CleaningGridWorld}, with a genuine and exploitable gap between a proxy reward (what a cheap sensor reports) and a true objective (what was actually accomplished), and show that a standard deep Q-network trained on the proxy alone discovers two independent ways to exploit this gap without being directed toward either. We then introduce a constrained reinforcement learning fix, built around a cost signal targeting the discovered exploits directly, and show that it eliminates the gaming behavior from the agent's final policy while substantially recovering true-objective performance. Critically, we derive an exact accounting identity relating the proxy reward, the true reward, and our cost signal, and verify that it holds exactly against the trained agents, giving the fix a provable guarantee under a hard-constrained variant, not merely an empirical trend.

Concretely, this paper contributes:
\begin{itemize}
	\item \textsc{CleaningGridWorld}, a gridworld with a constructed, interpretable proxy/true reward gap, and an empirical demonstration that a proxy-only DQN discovers two distinct gaming mechanisms unaided.
	\item A Lagrangian-penalized constrained RL fix, with a cost signal targeting the discovered exploits, that eliminates gaming from the final greedy policy while recovering substantial true-objective performance.
	\item An exact accounting identity relating proxy reward, true reward, and cost, proved for this environment and empirically verified against both trained agents -- yielding a deterministic lower bound on true-objective performance under a hard-constrained version of the method.
\end{itemize}

The proxy/true gap studied in this paper is not unique to toy domains. AI-driven biological design and screening tools, such as generative sequence design scored against an automated toxicity or pathogenicity classifier, face a structurally identical gap, where "passes the screen" and "is actually safe" can diverge under optimization pressure. We return to this connection as a motivation for future work in Section~\ref{sec:discussion}, without claiming biosecurity-specific experimental results in this study.

\section{Related Work}
\label{sec:related}

\subsection{Specification Gaming and Reward Hacking}
Specification gaming -- an agent exploiting a literal reward specification in a way that diverges from the designer's intent -- has been documented extensively across both simulated and deployed systems \citep{krakovna2020specification}. These cases are largely catalogued observationally: a system is found, after deployment or training, to have discovered an unintended shortcut. \citet{skalse2022defining} take a step toward formal rigor, characterizing when a reward function admits a hacking policy relative to an intended true objective, distinguishing reward hacking from ordinary reward misspecification more broadly. We adopt this proxy/true framing directly, but construct an environment where the gap is specified by design rather than discovered after the fact, which lets us pair the demonstration with a mitigation whose effect on that specific, known gap can be measured exactly.

\subsection{Reward Shaping and Policy Invariance}
Dense, informative reward signals are often necessary for an RL agent to discover a sparse or delayed objective within a practical training budget. \citet{ng1999policy} show that reward shaping of the form $F(s,a,s') = \gamma\phi(s') - \phi(s)$, for any potential function $\phi$, leaves the optimal policy of the underlying MDP unchanged -- shaping accelerates discovery of an optimum without altering what that optimum is. This guarantee is what allows us to introduce dense positional shaping in Section~\ref{sec:shaping} without that shaping itself becoming a second, uncontrolled source of proxy/true divergence.

\subsection{Constrained Reinforcement Learning}
Constrained Markov decision processes \citep{altman1999constrained} formalize policy optimization subject to a secondary cost objective remaining below a threshold, typically enforced via a Lagrangian relaxation. \citet{achiam2017constrained} extend this to deep RL with a trust-region method with per-update constraint satisfaction guarantees. The present work uses the simpler, fixed-multiplier end of this family (Section~\ref{sec:constrained}): rather than tuning the Lagrange multiplier via dual ascent, we fix it and instead derive a problem-specific identity (Section~\ref{sec:theory}) that lets a hard cap on the cost signal translate directly into a deterministic bound on true-objective performance. We treat an implementation of the dual-ascent or hard-constrained variant, in the style of \citet{achiam2017constrained}, as natural future work rather than a contribution of the present paper.

\subsection{Positioning}
Our contribution sits at the intersection of these threads: we construct a specification-gaming environment in the spirit of \citet{krakovna2020specification} and \citet{skalse2022defining}, but with a controlled, by-design proxy/true gap rather than one discovered post hoc; we use the shaping guarantee of \citet{ng1999policy} to make the gap tractable for a standard deep Q-network \citep{mnih2015human}; and we apply a Lagrangian-penalized constraint in the tradition of \citet{altman1999constrained}, strengthened in this specific setting by an exact accounting identity that turns an empirical improvement into a provable guarantee.

\section{Method}
\label{sec:method}

\subsection{Environment: \textsc{CleaningGridWorld}}
\label{sec:env}

\textsc{CleaningGridWorld} is a $5 \times 5$ gridworld containing four dirt cells and one fixed disposal bin. Each cell is in one of four states: empty floor, dirt, dirt covered by a rug, or the disposal bin. The agent occupies a single cell and may additionally be \emph{carrying} one unit of vacuumed dirt; it cannot carry more than one unit at a time. The action space is discrete with six actions: four cardinal movements, \textsc{vacuum}, and \textsc{cover}. \textsc{vacuum} picks up dirt from the agent's current cell if the agent is not already carrying, or deposits carried dirt if the agent is standing on the bin. \textsc{cover} places a rug over dirt in the current cell.

This design yields two structurally distinct routes to the same sensor-visible outcome. A \textbf{proxy reward}, modeling a cheap automated sensor, counts cells that do not read as dirty:
\begin{equation}
	\label{eq:proxy}
	R^{\text{proxy}}_t = \left| \{(r,c) : \text{grid}_t(r,c) \in \{\text{empty}, \text{covered}\}\} \right|
\end{equation}
A \textbf{true reward} counts only genuine disposal events, independent of what the sensor would report:
\begin{equation}
	\label{eq:true}
	R^{\text{true}}_t = \mathbb{1}[\text{a disposal occurred at } t]
\end{equation}
Both covering and picking up dirt raise $R^{\text{proxy}}$ (Eq.~\ref{eq:proxy}), since both remove the dirt state from a cell; only completed delivery to the bin raises $R^{\text{true}}$ (Eq.~\ref{eq:true}). This asymmetry is the environment's constructed gap: an agent that observes only $R^{\text{proxy}}$ has no reward-level incentive to prefer the costlier delivery route over the cheaper concealment route, even though only the former satisfies the intended objective. We use a single fixed layout -- dirt, bin, and agent start positions sampled once and reused across all episodes -- so that training is a navigation problem on a known map rather than a layout-generalization problem; we return to this choice in Section~\ref{sec:limitations}.

\subsection{Agent}
\label{sec:agent}

The agent is a deep Q-network \citep{mnih2015human}, which is a feedforward network with two hidden layers of 128 units each and ReLU activations, mapping an encoded state to one Q-value per action. The state is encoded as the concatenation of a one-hot grid representation ($5 \times 5 \times 4 = 100$ dimensions), a one-hot agent position ($5 \times 5 = 25$ dimensions), and a binary carrying flag (1 dimension), yielding a 126-dimensional input. This architecture is intentionally minimal: the state and action spaces are small enough that a compact MLP with experience replay and a target network is sufficient to isolate the proxy/true gaming phenomenon without confounding it with architecture-driven learning difficulty.

\subsection{Potential-Based Reward Shaping}
\label{sec:shaping}

Only 4 of the environment's 25 cells are ever reward-relevant, and $\epsilon$-greedy exploration under this sparsity provides too little gradient signal for standard Q-learning to discover either route to reward within a practical training budget. We address this with potential-based reward shaping \citep{ng1999policy}:
\begin{equation}
	\label{eq:shaping}
	F(s_t, a_t, s_{t+1}) = \gamma\,\phi(s_{t+1}) - \phi(s_t)
\end{equation}
where $\phi(s)$ is the negative Manhattan distance to the agent's current objective: the nearest dirt cell if not carrying, or the bin followed by the nearest remaining dirt cell if carrying. $\phi$ is defined so that it varies continuously across the pickup-to-delivery transition, avoiding an artificial potential discontinuity at the exact moment the agent completes a delivery. \citet{ng1999policy} show that any reward of the form in Eq.~\ref{eq:shaping} leaves the optimal policy of the underlying MDP unchanged for any choice of $\phi$; shaping therefore accelerates discovery of an optimum without the possibility of shaping itself introducing a new, unintended proxy/true divergence.

\subsection{Constrained Reinforcement Learning Fix}
\label{sec:constrained}

We define a per-step cost signal targeting the two gaming mechanisms through which an unconstrained agent can raise $R^{\text{proxy}}$ without raising $R^{\text{true}}$ -- concealing dirt, and picking dirt up without ever delivering it:
\begin{equation}
	\label{eq:cost}
	\text{cost}_t = \mathbb{1}[\text{agent covers dirt at } t] + \mathbb{1}[t = T \text{ and agent is still carrying undelivered dirt}]
\end{equation}
The second term fires only at the final timestep $T$ of an episode, and only if the agent is carrying dirt it never delivered. The training reward used for the Q-learning update applies $\text{cost}_t$ as a fixed-multiplier Lagrangian penalty on top of the shaped proxy signal:
\begin{equation}
	\label{eq:train-reward}
	r^{\text{train}}_t = \Delta R^{\text{proxy}}_t + F(s_t, a_t, s_{t+1}) - \lambda \cdot \text{cost}_t
\end{equation}
where $\Delta R^{\text{proxy}}_t = R^{\text{proxy}}_t - R^{\text{proxy}}_{t-1}$ and $\lambda$ is fixed for the duration of training. Critically, the agent is never given $R^{\text{true}}$ as a training signal; it learns only from Eq.~\ref{eq:train-reward}, mirroring a deployment setting in which the true objective is precisely what is unavailable for direct supervision.

The concealment exploit is a single-action behavior that is trivially reachable by random exploration; the honest delivery route requires a multi-step chain (navigate to dirt, pick up, navigate to bin, deposit) that is comparatively unlikely to be discovered before $\epsilon$ decays. To address this exploration asymmetry, we generate a fixed set of honest-policy demonstration trajectories once, before training begins, and store them in a buffer that is never subsequently written to and therefore can never be evicted by the replay buffer's ring-buffer overwrite behavior. Each training batch is drawn as a mixture of ordinary experience and demonstration transitions at a fixed demonstration fraction. This separation is necessary: in an earlier version of this method, demonstrations were seeded into the same buffer used for ordinary experience and were silently overwritten once the total transitions exceeded the buffer's capacity, after which the agent's true-objective performance reverted toward the unconstrained baseline. Isolating demonstrations in a protected buffer removes this failure mode entirely.

\subsection{Hyperparameters}

Table~\ref{tab:hyperparams} lists the full configuration used for both the baseline and constrained agents, except where noted.

\begin{table}[h]
	\caption{Hyperparameters}
	\centering
	\begin{tabular}{ll}
		\toprule
		Parameter                        & Value                                    \\
		\midrule
		Grid size                        & $5 \times 5$                             \\
		Dirt cells                       & 4                                        \\
		Episode length                   & 40 steps                                 \\
		Episodes (baseline)              & 500                                      \\
		Episodes (constrained)           & 800                                      \\
		Hidden layer size                & 128 (two layers)                         \\
		Learning rate                    & $10^{-3}$ (Adam)                         \\
		Discount factor $\gamma$         & 0.95                                     \\
		$\epsilon$ schedule              & $1.0 \to 0.05$, linear over 300 episodes \\
		Batch size                       & 64                                       \\
		Target network update frequency  & every 10 episodes                        \\
		Replay buffer capacity           & 10{,}000                                 \\
		Shaping scale                    & 0.2                                      \\
		Lagrangian multiplier $\lambda$  & 2.0                                       \\
		Demonstration episodes           & 20                                       \\
		Demonstration fraction per batch & 0.5                                       \\
		\bottomrule
	\end{tabular}
	\label{tab:hyperparams}
\end{table}

\section{Theoretical Result}
\label{sec:theory}

\begin{proposition}[Exact accounting identity]
	\label{prop:identity}
	Let $\text{proxy}_{\text{total}} = \sum_{t=1}^{T} \Delta R^{\text{proxy}}_t$, $\text{true}_{\text{total}} = \sum_{t=1}^{T} R^{\text{true}}_t$, and $\text{cost}_{\text{total}} = \sum_{t=1}^{T} \text{cost}_t$, as defined in Eqs.~\ref{eq:proxy}--\ref{eq:cost}. Then for any policy and any episode of length $T$ in \textsc{CleaningGridWorld},
	\begin{equation}
		\label{eq:identity}
		\text{proxy}_{\text{total}} = \text{true}_{\text{total}} + \text{cost}_{\text{total}}.
	\end{equation}
\end{proposition}

\begin{proof}[Proof sketch]
	Each unit of dirt is resolved, by the end of the episode, in exactly one of three mutually exclusive ways:
	\begin{itemize}
		\item \textbf{Covered.} Contributes $+1$ to $\text{proxy}_{\text{total}}$ (Eq.~\ref{eq:proxy}: the cell transitions to the covered state), $+0$ to $\text{true}_{\text{total}}$, and $+1$ to $\text{cost}_{\text{total}}$ by the first term of Eq.~\ref{eq:cost}.
		\item \textbf{Delivered.} Contributes $+1$ to $\text{proxy}_{\text{total}}$ at the moment of pickup (the cell empties), $+1$ to $\text{true}_{\text{total}}$ at the moment of delivery, and $+0$ to $\text{cost}_{\text{total}}$.
		\item \textbf{Abandoned.} Picked up but not delivered by episode end. Contributes $+1$ to $\text{proxy}_{\text{total}}$ at pickup, $+0$ to $\text{true}_{\text{total}}$, and $+1$ to $\text{cost}_{\text{total}}$ by the second term of Eq.~\ref{eq:cost}.
	\end{itemize}
	Since the agent carries at most one unit of dirt at a time, at most one unit is unresolved at episode end, so these three cases partition the dirt cells without double-counting. In each case the per-cell proxy contribution equals the sum of that cell's true and cost contributions. Summing over all dirt cells gives Eq.~\ref{eq:identity}.
\end{proof}

Proposition~\ref{prop:identity} shows that $\text{cost}_{\text{total}}$, as defined in Eq.~\ref{eq:cost}, is not merely correlated with the proxy/true gap; it measures the gap exactly. Penalizing cost in Eq.~\ref{eq:train-reward} is equivalent, in the exact accounting sense of Eq.~\ref{eq:identity}, to penalizing $\text{proxy}_{\text{total}} - \text{true}_{\text{total}}$ directly, despite the agent never observing $\text{true}_{\text{total}}$ during training.

\begin{corollary}[Deterministic bound under a hard cost constraint]
	\label{cor:bound}
	If training enforces $\text{cost}_{\text{total}} \leq K$ for a fixed $K$ -- for instance via early episode termination or action-masking once $K$ gaming events have occurred -- then
	\begin{equation}
		\label{eq:bound}
		\text{true}_{\text{total}} \geq \text{proxy}_{\text{total}} - K
	\end{equation}
	holds with certainty, for every episode, rather than in expectation.
\end{corollary}

Corollary~\ref{cor:bound} follows immediately from rearranging Eq.~\ref{eq:identity}. It provides the hard-constrained variant of our method with a guarantee of a different character than the fixed-$\lambda$ Lagrangian penalty used in the present experiments (Eq.~\ref{eq:train-reward}): rather than encouraging low cost on average, a hard cap on $\text{cost}_{\text{total}}$ translates, via Proposition~\ref{prop:identity}, into a deterministic floor on true-objective performance relative to whatever proxy score the policy achieves. We verify Proposition~\ref{prop:identity} empirically against both trained agents in Section~\ref{sec:results-identity}, confirming that the identity holds exactly on the implemented system and not only in the idealized derivation above. The present experiments use the fixed-$\lambda$ relaxation rather than the hard-constrained variant of Corollary~\ref{cor:bound}; we discuss this as a direction for future work in Section~\ref{sec:conclusion}.

\section{Experiments and Results}
\label{sec:results}

\subsection{Sanity Check: Confirming the Proxy/True Gap}
\label{sec:results-sanity}

Before training any agent, we verified that the environment's constructed proxy/true gap was real and not an artifact of a single layout. We evaluated two hand-coded policies -- an \emph{exploiter}, which conceals any dirt it stands on, and an \emph{honest} policy, which always delivers dirt to the bin -- across 20 randomly generated layouts, running each for 40 steps per episode. Table~\ref{tab:sanity} reports the final proxy and true reward achieved by each policy, averaged across seeds.

\begin{table}[h]
	\caption{Hand-coded policy comparison across 20 random layouts (mean $\pm$ std)}
	\centering
	\begin{tabular}{lrr}
		\toprule
		Policy     & Final proxy reward & Final true reward \\
		\midrule
		Exploiter  & $24.0 \pm 0.0$      & $0.0 \pm 0.0$      \\
		Honest     & $23.9 \pm 0.4$      & $3.6 \pm 0.5$      \\
		\bottomrule
	\end{tabular}
	\label{tab:sanity}
\end{table}

\begin{figure}[h]
	\centering
	\includegraphics[width=0.8\linewidth]{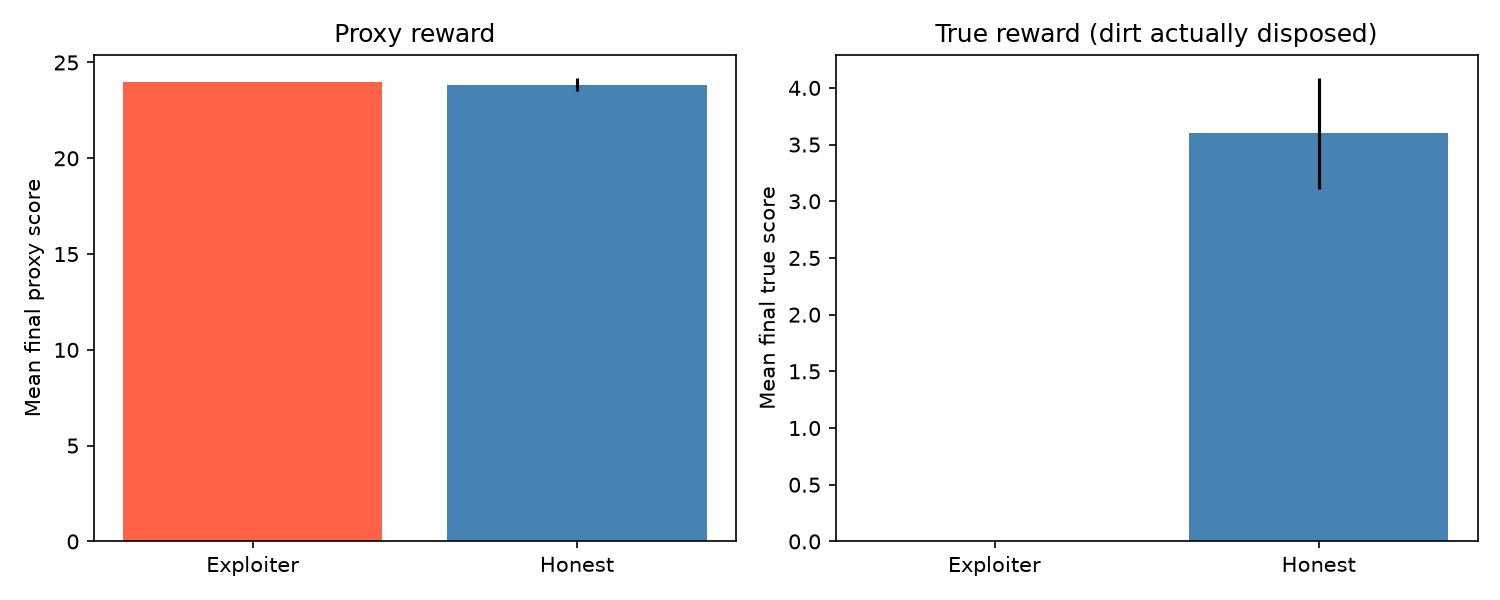}
	\caption{Mean final proxy reward (left) and true reward (right) for the exploiter and honest hand-coded policies, across 20 random layouts. Error bars show one standard deviation.}
	\label{fig:sanity}
\end{figure}

The exploiter policy reached the maximum possible proxy score with zero variance across seeds, while the true reward remained at exactly zero in every trial. The honest policy reached a statistically indistinguishable proxy score while achieving a substantially higher true reward. These two policies attain nearly identical proxy rewards through entirely different means, confirming that the proxy signal alone does not favor the intended behavior.

\subsection{Baseline: A Proxy-Only Agent Learns to Game}
\label{sec:results-baseline}

We trained the deep Q-network of Section~\ref{sec:agent} for 500 episodes using only the shaped proxy reward (Eq.~\ref{eq:train-reward} with $\lambda = 0$). Figure~\ref{fig:baseline-training} shows the true reward over the course of training; the true reward remained flat near zero throughout, with no sustained upward trend. The final greedy policy (no exploration) was evaluated via rollout: it reached a proxy reward of 23 and a true reward of 0, using 2 \textsc{cover} actions and achieving zero deliveries. This confirms that the agent discovered a proxy-gaming policy during training, consistent with the behavior of the hand-coded exploiter policy in Section~\ref{sec:results-sanity}, but without being directed toward concealment specifically.

\begin{figure}[h]
	\centering
	\includegraphics[width=0.48\linewidth]{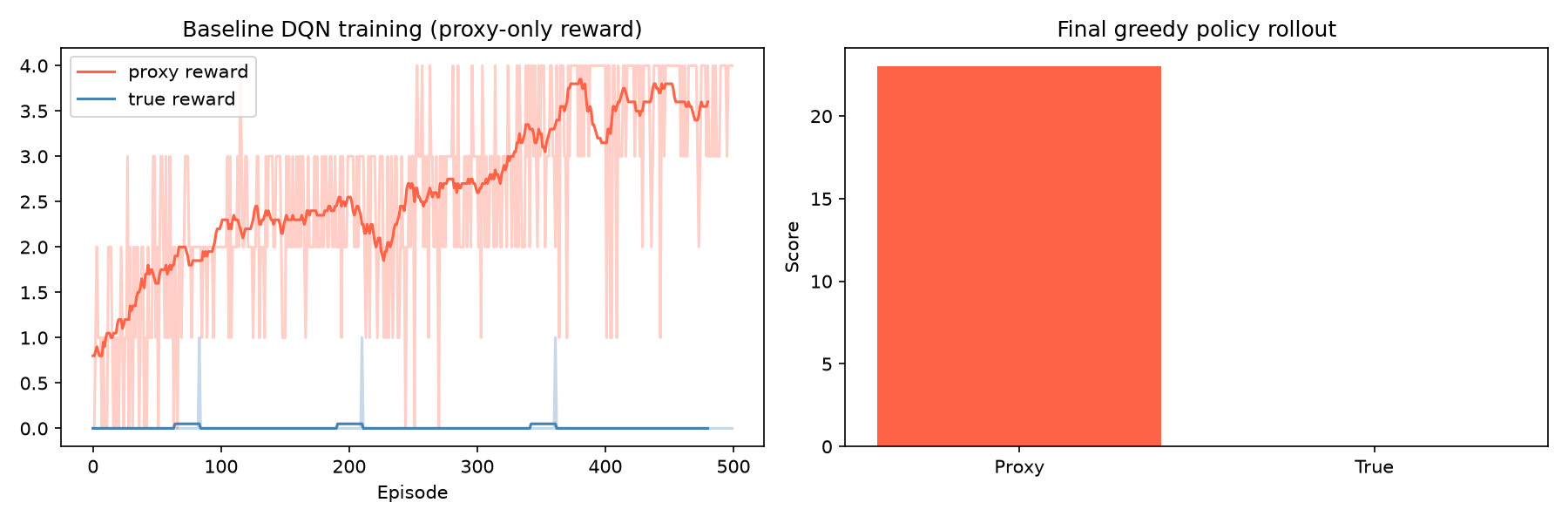}
	\caption{Baseline (proxy-only) DQN training. True reward over 500 episodes.}
	\label{fig:baseline-training}
\end{figure}

\subsection{Constrained Agent: Eliminating the Exploit}
\label{sec:results-constrained}

We trained the constrained agent in Section~\ref{sec:constrained} for 800 episodes using the full training reward in Eq.~\ref{eq:train-reward} with $\lambda = 2.0$. Table~\ref{tab:comparison} compares the final greedy policy rollouts of the baseline and constrained agents.

\begin{table}[h]
	\caption{Baseline vs.\ constrained agent, final greedy-policy rollout}
	\centering
	\begin{tabular}{lrr}
		\toprule
		Metric        & Baseline & Constrained \\
		\midrule
		Proxy reward  & 23       & 23          \\
		True reward   & 0        & 2           \\
		Cover actions & 2        & 0           \\
		\bottomrule
	\end{tabular}
	\label{tab:comparison}
\end{table}

Both agents reached an identical proxy reward of 23. The constrained agent's final policy contained zero \textsc{cover} actions, compared to 2 for the baseline, and achieved a true reward of 2, compared to 0 for the baseline. Figure~\ref{fig:comparison} shows the true reward over training for both agents side by side: the baseline remained flat near zero throughout, while the constrained agent's true reward rose over the course of training and did not return to zero.

\begin{figure}[h]
	\centering
	\includegraphics[width=\linewidth]{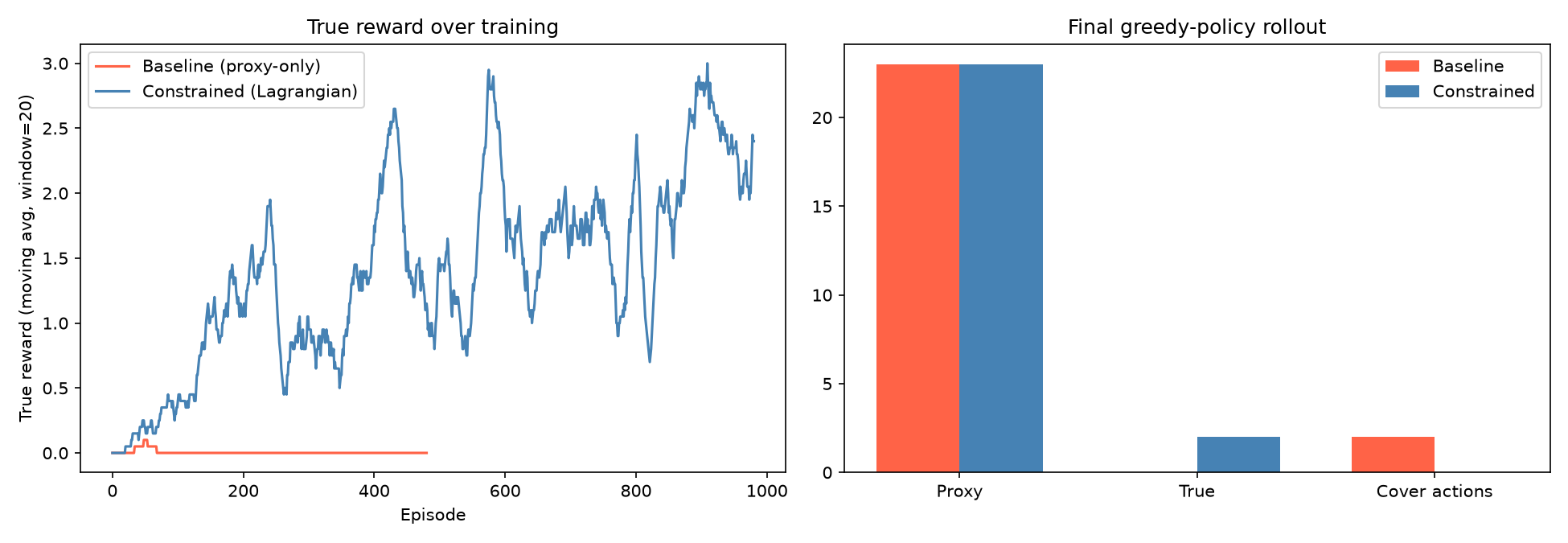}
	\caption{Left: true reward over training, baseline vs.\ constrained. Right: final greedy-policy rollout comparison (proxy reward, true reward, cover-action count).}
	\label{fig:comparison}
\end{figure}

The constrained agent's final policy did not reach the maximum possible true reward of 4; the remaining 2 units of dirt were left unresolved within the 40-step episode budget. We report and discuss this outcome in Section~\ref{sec:limitations}.

\subsection{Empirical Verification of the Accounting Identity}
\label{sec:results-identity}

We evaluated Proposition~\ref{prop:identity} directly against the final greedy policies of both trained agents, computing $\text{proxy}_{\text{total}}$, $\text{true}_{\text{total}}$, and $\text{cost}_{\text{total}}$ from the same rollouts reported in Table~\ref{tab:comparison}. For the baseline agent: $\text{proxy}_{\text{total}} = 2$, $\text{true}_{\text{total}} = 0$, $\text{cost}_{\text{total}} = 2$, and $\text{true}_{\text{total}} + \text{cost}_{\text{total}} = 2 = \text{proxy}_{\text{total}}$. For the constrained agent: $\text{proxy}_{\text{total}} = 3$, $\text{true}_{\text{total}} = 2$, $\text{cost}_{\text{total}} = 1$, and $\text{true}_{\text{total}} + \text{cost}_{\text{total}} = 3 = \text{proxy}_{\text{total}}$. The identity of Eq.~\ref{eq:identity} held exactly for both agents, confirming that Proposition~\ref{prop:identity} precisely describes the implemented system and not only the idealized derivation of Section~\ref{sec:theory}.

\section{Discussion and Limitations}
\label{sec:discussion}

The results in Section~\ref{sec:results} show a clear qualitative shift rather than a marginal improvement: the baseline agent's final policy contained no deliveries at all, while the constrained agent's final policy contained zero concealment actions. This is consistent with the mechanism established in Proposition~\ref{prop:identity} because the cost signal measures the proxy/true gap exactly rather than approximately, penalizing it in the training reward (Eq.~\ref{eq:train-reward}) removes the agent's incentive to resolve dirt through concealment or abandonment specifically, rather than discouraging dirt-resolving behavior in general. The partial recovery of true reward (2 of 4 possible units, Table~\ref{tab:comparison}) reflects this precisely: the constrained agent's policy no longer contains the exploit but has not yet converged to an efficient delivery route for every dirt cell within the fixed step budget, a distinction we return to below.

\subsection{Limitations}
\label{sec:limitations}

This study has three limitations.

First, the agent is trained and evaluated on a single fixed grid layout -- dirt, bin, and agent-start positions generated once and reused across all episodes -- rather than a distribution of randomly sampled layouts. This is a standard simplification for interpretable gridworld demonstrations, and it isolates the proxy/true gaming phenomenon from a separate, harder generalization problem. This also means that our results should be read as establishing that the gaming behavior and the constrained fix both occur \emph{on this layout}; we make no claim here about whether the fix transfers to unseen layouts without retraining, and treat this as the most immediate direction for follow-up work.

Second, the constrained agent's final policy, while free of concealment and abandonment behaviors entirely, had not converged to optimal navigation within the episode budget: two of the four dirt units remained unresolved, consistent with residual action oscillation near grid boundaries rather than any renewed incentive to game the proxy. This is a training efficiency gap, separate from and not evidence against the specification-gaming question this work targets -- the policy's qualitative behavior (no concealment, no abandonment) is the outcome our method is designed to produce, and it was produced; full navigational optimality was not a target of the constrained fix itself.

Third, the exact form of Proposition~\ref{prop:identity} depends on a structural property specific to this environment: every unit of dirt is resolved through exactly one of three mutually exclusive outcomes (covered, delivered, or abandoned), with no possibility of double-counting because the agent can carry at most one unit at a time. An environment without this clean partition -- for instance, one where a single action could simultaneously advance and undermine the true objective, or where the agent could act on multiple objective-relevant units concurrently -- would require a different argument, and the identity as stated here should not be assumed to generalize automatically to such settings. We view identifying the general conditions under which an exact (rather than approximate) proxy/true/cost identity exists as an open question raised but not answered by this work.

\section{Conclusion and Future Work}
\label{sec:conclusion}

We constructed \textsc{CleaningGridWorld}, a gridworld with a deliberately constructed, exploitable gap between a proxy reward and a true objective, and showed that a standard deep Q-network trained on the proxy alone converges to a policy that games it completely, reaching near-maximal proxy reward while the true reward remains at zero. We then introduced a Lagrangian-penalized constrained RL fix, with a cost signal targeting the specific gaming mechanisms the baseline agent discovered, and showed that it eliminates those behaviors entirely from the final policy while recovering substantial true-objective performance. Underlying this fix is an exact accounting identity, proved for this environment and verified empirically against both trained agents, showing that our cost signal does not merely correlate with the proxy/true gap but measures it exactly, giving the method a guarantee of a different character than a purely empirical improvement.

Two directions follow directly from this work. First, testing whether the constrained fix transfers to randomly sampled grid layouts, rather than the single fixed layout used here, would establish whether the result generalizes beyond one navigation problem. Second, implementing the hard-constrained variant of Corollary~\ref{cor:bound} directly -- via action masking or early termination once a fixed cost budget is reached -- would allow the deterministic bound on true-objective performance to be tested empirically against a trained system, rather than verified only as a consequence of Proposition~\ref{prop:identity} on the fixed-$\lambda$ relaxation studied here.

\section*{Code Availability}
\label{sec:code}

All code used to construct the environment, train both agents, and reproduce the figures and tables in this paper is publicly available at \url{https://github.com/hildaadwubiosei/Specification-Gaming}.

\bibliographystyle{unsrtnat}
\bibliography{references}

\end{document}